\documentclass[a4paper,UKenglish,cleveref,autoref,thm-restate]{lipics-v2021}

\title{Parameterized Complexity of Logic-Based Argumentation in Schaefer's Framework} 

\author{Yasir Mahmood}{Leibniz Universit\"at Hannover, Institut f\"ur Theoretische Informatik, Germany}{mahmood@thi.uni-hannover.de}{https://orcid.org/0000-0002-5651-5391}{}

\author{Arne Meier}{Leibniz Universit\"at Hannover, Institut f\"ur Theoretische Informatik, Germany}{meier@thi.uni-hannover.de}{https://orcid.org/0000-0002-8061-5376}{}

\author{Johannes Schmidt}{Jönköping University, Department of Computer Science and Informatics, School of Engineering, Sweden}{johannes.schmidt@ju.se}{https://orcid.org/0000-0001-8551-1624}{}

\authorrunning{Y.~Mahmood, A.~Meier, and J.~Schmidt}
\Copyright{Yasir Mahmood, Arne Meier, and Johannes Schmidt}


\funding{This work was supported by the German Research Foundation (DFG), project ME 4279/1-2.}
\acknowledgements{The authors thank the anonymous referees for their valuable feedback.} 

\nolinenumbers 

\hideLIPIcs  

\EventEditors{John Q. Open and Joan R. Access}
\EventNoEds{2}
\EventLongTitle{42nd Conference on Very Important Topics (CVIT 2016)}
\EventShortTitle{CVIT 2016}
\EventAcronym{CVIT}
\EventYear{2016}
\EventDate{December 24--27, 2016}
\EventLocation{Little Whinging, United Kingdom}
\EventLogo{}
\SeriesVolume{42}
\ArticleNo{23}

\newcommand{\complClFont}[1]{\mathbf{#1}}         
\newcommand{\problemFont}[1]{\mathrm{#1}}         
\newcommand{\mathCommandFont}[1]{\mathrm{#1}}     
\newcommand{\cloneFont}[1]{\mathsf{#1}}     

\newcommand{\ARG}{\protect\ensuremath{\problemFont{ARG}}} 
\newcommand{\ARGcheck}{\protect\ensuremath{\problemFont{ARG\text-Check}}}
\newcommand{\ARGrel}{\protect\ensuremath{\problemFont{ARG\text-Rel}}}

\newcommand{\argument}{\protect\ensuremath{(\Phi, \alpha)}}
\newcommand{\threeCNF}{\protect\ensuremath{\problemFont{3CNF}}}

\newcommand{\SAT}{\protect\ensuremath{\problemFont{SAT}}}

\newcommand{\IMP}{\protect\ensuremath{\problemFont{IMP}}}
\newcommand{\pIMP}{\protect\ensuremath{\p\problemFont{IMP}}}
\newcommand{\p}{\problemFont{p\text-}}

\newcommand{\clique}{\protect\ensuremath{\textsc{Clique}}\xspace}

\newcommand{\wsat}[2]{\protect\ensuremath{\p\problemFont{WSAT}(\Gamma_{#1,#2})}} 
\newcommand{\WSAT}{\protect\ensuremath{\p\problemFont{WSAT}}} 

\newcommand{\FPT}{\protect\ensuremath{\complClFont{FPT}}\xspace}

\newcommand{\W}[1]{\protect\ensuremath{\complClFont{W}\ifx#1\empty\else[#1]\fi}\xspace}

\renewcommand{\P}{\protect\ensuremath{\complClFont{P}}\xspace}

\newcommand{\para}{\protect\ensuremath{\complClFont{para\text-}}\xspace}
\newcommand{\NP}{\protect\ensuremath{\complClFont{NP}}\xspace}
\newcommand{\Ptime}{\protect\ensuremath{\complClFont{P}}\xspace}
\newcommand{\DP}{\protect\ensuremath{\complClFont{DP}}\xspace}
\newcommand{\SigmaP}{\protect\ensuremath{\complClFont{\Sigma_2^P}}\xspace}
\newcommand{\co}{\protect\ensuremath{\complClFont{co}}}

\newcommand{\IS}[2]{\protect\ensuremath{\cloneFont{IS}^{#1}_{#2}}}
\newcommand{\ID}{\protect\ensuremath{\cloneFont{ID}}}

\newcommand{\IM}{\protect\ensuremath{\cloneFont{IM}}}

\newcommand{\leqlogm}{\leq^{\mathCommandFont{log}}_m}

\newcommand{\fptreduction}{\leq^{\mathCommandFont{\FPT}}}
\newcommand{\preduction}{\leq^{\mathCommandFont{\P}}_m}

\newcommand{\setdefinition}[1]{\protect\ensuremath{\{\, #1 \,\} }}
\newcommand{\clos}[1]{\left\langle #1 \right\rangle}
\newcommand{\closneq}[1]{\left\langle #1 \right\rangle_{\neq}}
\newcommand{\closnexneq}[1]{\left\langle #1 \right\rangle_{\not\exists,\neq}}

\newcommand{\csl}{\protect\ensuremath{\Gamma}}
\newcommand{\cocl}{\protect\ensuremath{\cloneFont{C}}}
\newcommand{\enc}[1]{\protect\ensuremath{\mathrm{enc}(#1)}}
\newcommand{\Vars}[1]{\protect\ensuremath{\mathrm{Vars}(#1)}}

\newcommand{\paraproblemdef}[4]{%
The problem #1 asks, given #2, parameterized by #3, #4?
}

\newcommand{\problemdef}[3]{%
The problem #1 asks, given #2, #3?
}

\newcommand{\dfn}{\mathrel{\mathop:}=}
\newcommand{\T}{\protect\ensuremath{\mathrm{T}}\xspace}
\newcommand{\F}{\protect\ensuremath{\mathrm{F}}\xspace}
\newcommand{\PosOneThreeSAT}{\protect\ensuremath\problemFont{Pos}\text{-1-}\problemFont{In}\text{-}3\text{-}\problemFont{Sat}}
\usepackage{xspace}
\usepackage{amsthm}
\usepackage{todonotes}

{%
  \addtocounter{theorem}{-1}
  \endgroup
}%

\begin{document}

\maketitle
\begin{abstract}
	Logic-based argumentation is a well-established formalism modelling nonmonotonic reasoning.
	It has been playing a major role in AI for decades, now. 
	Informally, a set of formulas is the support for a given claim if it is consistent, subset-minimal, and implies the claim.
	In such a case, the pair of the support and the claim together is called an argument. 
	In this paper, we study the propositional variants of the following three computational tasks studied in argumentation: ARG (exists a support for a given claim with respect to a given set of formulas), ARG-Check (is a given set a support for a given claim), and ARG-Rel (similarly as ARG plus requiring an additionally given formula to be contained in the support).
	ARG-Check is complete for the complexity class DP, and the other two problems are known to be complete for the second level of the polynomial hierarchy (Parson~et~al., J. Log. Comput., 2003) and, accordingly, are highly intractable.
	Analyzing the reason for this intractability, we perform a two-dimensional classification:
	first, we consider all possible propositional fragments of the problem within Schaefer's framework (STOC 1978), and then study different parameterizations for each of the fragment.
	We identify a list of reasonable structural parameters (size of the claim, support, knowledge-base) that are connected to the aforementioned decision problems.
	Eventually, we thoroughly draw a fine border of parameterized intractability for each of the problems showing where the problems are fixed-parameter tractable and when this exactly stops.
	Surprisingly, several cases are of very high intractability (paraNP and beyond).
\end{abstract}

\section{Introduction}
Argumentation is a nonmonotonic formalism in artificial intelligence around which an active research community has evolved \cite{DBLP:journals/aim/AtkinsonBGHPRST17,DBLP:journals/ai/AmgoudP09,DBLP:conf/ijcai/RagoCT18,handbookargu}.
Essentially, there exist two branches of argumentation: the \emph{abstract} \cite{DBLP:journals/ai/Dung95} and the \emph{logic-based} \cite{DBLP:journals/ai/BesnardH01,DBLP:books/mit/BH2008,DBLP:journals/csur/ChesnevarML00,Prakken2002} approach.
The abstract setting mainly focusses on formalizing the argumentative structure in a graph-theoretic way.
Arguments are nodes in a directed graph and the `attack-relation' draws which argument eliminates which other.
In the logic-based method, one looks for inclusion-minimal consistent sets of formulas $\Phi$ (the \emph{support}) that entail a \emph{claim} $\alpha$, modelled through a formula (in the positive case one calls $(\Phi,\alpha)$ an \emph{argument}).
In this paper, we focus on the latter formalism and, specifically, study three decision problems.
The first, $\ARG$, asks, given a set of formulas $\Delta$ (the \emph{knowledge-base}) and a formula $\alpha$, whether there exists a subset $\Phi\subseteq\Delta$ such that $(\Phi,\alpha)$ is an argument in $\Delta$.
The two further problems of interest are $\ARGcheck$ (is a given set a support for a given claim), and $\ARGrel$ ($\ARG$ plus requiring an additionally given formula to be contained in the support, too).
\begin{example}[\cite{DBLP:journals/ai/BesnardH01}]
	Consider the following two arguments.
	(A1) Support: Donald is a public person, so we can publicize details about his private life.
		Claim: We can publicize that Donald plays golf. 
	(A2) Support: Donald just resigned from politics; as a result, he is no longer a public person.
		Claim: Donald is no longer a public person.\smallskip
		
	Formalizing these arguments would yield
		$A_1: \Phi_1=\{x_{pd}\to x_{dg}, x_{pd}\}, \alpha_1=\{x_{dg}\}$,
		$A_2: \Phi_2=\{x_{rd}\to \lnot x_{pd},x_{rd}\}, \alpha_2=\{\lnot x_{pd}\}$,
	where $x_{pd}\triangleq$ ``Donald is a public person'', $x_{dg}\triangleq$``Donald plays golf'', $x_{rd}\triangleq$``Donald resigned from politics''.
	Each argument is supporting its claim, yet together they are conflicting, as $A_2$ attacks $A_1$.
\end{example}
It is rather computationally involved to compute the support of an argument, as $\ARG$ was shown to be $\SigmaP$-complete by Parsons~et~al.~\cite{DBLP:journals/logcom/ParsonsWA03}.
Yet, there have been made efforts to improve the understanding of this high intractability by Creignou~et~al.~\cite{DBLP:journals/tocl/CreignouE014,DBLP:journals/argcom/Creignou0TW11} in two settings: Schaefer's \cite{DBLP:conf/stoc/Schaefer78} as well as Post's \cite{pos41} framework.
Clearly, such research aims for drawing the fine intractability frontier of computationally involved problems to show for what restrictions there still is hope to reach algorithms running for practical applications.
Both approaches mainly focus on restrictions on the logical part of the problem language, that is, restricting the allowed connectives or available constraints.

In this paper, Schaefer's approach is our focus, that is, the formulas we study are propositional formulas in conjunctive normalform (CNF) whose clauses are formed depending on a fixed set of relations $\Gamma$ (the so-called \emph{constraint language}, short CL).
In this setting, Schaefer's framework~\cite{DBLP:conf/stoc/Schaefer78} captures well-known classes of CNF-formulas (e.g., Horn, dual-Horn, or Krom).
Accordingly, one can see classifications in such a setting as a one-dimensional approach (the dimension is given rise by the considered logical fragments).

We consider a second dimension on the problem in this paper, namely, by investigating its parameterized complexity~\cite{DBLP:series/txcs/DowneyF13}.
Motivated by the claim that the input length is not the only important structural aspect of problems, one studies so-called parameterizations (or parameters) of a problem.
The goal of such studies is to identify a parameter that is relevant for practice but also is slowly growing or even of constant value. 
If, additionally, one is able to construct an algorithm that solves the problem in time $f(k)\cdot|x|^{O(1)}$ for some computable function $f$ and all inputs $(x,k)$, then one calls the problem \emph{fixed-parameter tractable}. 
That is why in this case one can solve the problem (for fixed parameter values) in polynomial time.
As a result, this complexity class is seen to capture the idea of efficiency in the parameterized sense.
While $\NP$-complete problems are considered intractable in the classical setting, on the parameterized level, the complexity class $\W1$ is seen to play this counterpart.
Informally, this class is characterized via a special kind of satisfiability questions.
Above this class an infinite $\W{}$-hierarchy is defined which culminates in the class $\W\P$, which in turn is contained in the class $\para\NP$ (problems solvable by NTMs in time $f(k)\cdot|x|^{O(1)}$).

\paragraph*{Contributions} 
Our main contributions are the following.
\begin{enumerate}
	\item We initiate a thorough study of the parameterized complexity of logic-based argumentation.
	We study three parameters: size of the support, of the claim, and of the knowledge-base.
	We show that the complexity of $\ARG$, regarding the claim as a parameter, varies: $\FPT$, $\W1$-, $\W2$-, $\para\NP$-, $\para\co\NP$-, as well as $\para\SigmaP$-complete cases.
	For the same parameter, $\ARGcheck$ is $\FPT$ for Schaefer, $\para\DP$-complete otherwise.
	$\ARGrel$ is $\FPT$, $\para\NP$-, or $\para\SigmaP$-complete.

	The size of the knowledge-base as the parameter yields dichotomy results for the two problems $\ARG$ and $\ARGrel$: $\FPT$ versus membership in $\para\co\NP$ and a lower bound that relates to the implication problem.

	Concerning the size of the support as the parameter, we prove a dichotomy: $\FPT$ versus $\para\DP$-membership and the same hardness as the implication problem.
	\item As a byproduct, we advance the algebraic tools in the context of Schaefer's framework, and show a list of technical implementation results that are independent of the studied problem and might be beneficial for further research in the constraint context.
	\item We classify the parameterized complexity of the implication problem (does a set of propositional formulas $\Phi$ imply a propositional formula $\alpha$?) with respect to the parameter $|\alpha|$ and show that it is $\FPT$ if the CL is Schaefer, and $\para\co\NP$-complete otherwise. 
\end{enumerate}

\paragraph*{Related Work}
Very recently, Mahmood~et~al.~\cite{DBLP:conf/lfcs/0002M020} presented a parameterized classification of abductive reasoning in Schaefer's framework. 
Some of their cases, as well as results from Nordh and Zanuttini~\cite{DBLP:journals/ai/NordhZ08} relate to some of our results.
The studies of the implication problem in the frameworks of Schaefer~\cite{DBLP:conf/dagstuhl/SchnoorS08} as well as in the one in Post~\cite{DBLP:journals/ipl/BeyersdorffMTV09} prove a classical complexity landscape.
Last year, Hecher~et~al.~\cite{DBLP:conf/aaai/FichteHM19} conducted a parameterized study of abstract argumentation.
The known classical results \cite{DBLP:journals/tocl/CreignouE014,DBLP:journals/ai/NordhZ08,DBLP:conf/dagstuhl/SchnoorS08,DBLP:journals/ipl/BeyersdorffMTV09} are partially used in some of our proofs, e.g., showing some parameterized complexity lower bounds.
The two mentioned parameterized complexity related papers \cite{DBLP:conf/aaai/FichteHM19,DBLP:conf/lfcs/0002M020} both are about different formalisms that are slightly related to our setting (the first is about abstract argumentation, the second on abduction). 

\section{Preliminaries}
We assume familiarity with basic notions in complexity theory (cf.~\cite{DBLP:books/daglib/0086373}) and use the complexity classes $\Ptime, \NP, \co\NP,\SigmaP$.
For a set $S$, we write $|S|$ for its \emph{cardinality}.
Abusing notation, we will use $|w|$, for a string $w$, to denote its \emph{length}.
If $\varphi$ is a formula, then $\Vars\varphi$ denotes its set of variables, and $\enc \varphi$ its \emph{encoding}.
W.l.o.g., we assume a reasonable encoding computable in polynomial time that encodes variables in binary. 
The \emph{weight} of an assignment $\sigma$ is the number of variables mapped to $1$.

\paragraph*{Parameterized Complexity} We give a brief introduction to parameterized complexity theory.
A more detailed exposition can be found in the textbook of Downey and Fellows~(\cite{DBLP:series/txcs/DowneyF13}).
A \emph{parameterized problem (PP)} $\Pi$ is a subset of $\Sigma^*\times\mathbb N$, where $\Sigma$ is an alphabet.
For an instance $(x,k)\in\Sigma^*\times\mathbb N$, $k$ is called the \emph{parameter}.
If there exists a deterministic algorithm deciding $\Pi$ in time $f(k)\cdot|x|^{O(1)}$ for every input $(x,k)$, where $f$ is a computable function, then $\Pi$ is \emph{fixed-parameter tractable} (short: $\FPT$).

\begin{definition}
	Let $\Sigma$ and $\Delta$ be two alphabets.
	 A PP $\Pi\subseteq\Sigma^*\times\mathbb{N}$ \emph{fpt-reduces} to a PP $\Theta\subseteq\Delta^*\times\mathbb N$, in symbols $\Pi\fptreduction\Theta$, if the following is true:
	(i) there is an $\FPT$-computable function $f$, such that, for all $(x,k)\in\Sigma^*\times\mathbb N$: $(x,k)\in \Pi\Leftrightarrow f(x,k)\in \Theta$,
	(ii) there exists a computable function $g\colon\mathbb N\to\mathbb N$ such that for all $(x,k)\in\Sigma^*\times\mathbb N$ and $f(x,k)=(y,\ell)$: $\ell\leq g(k)$.
\end{definition}

The problems $\Pi$ and $\Theta$ are $\FPT$-equivalent if both $\Pi\fptreduction\Theta$ and $\Theta\fptreduction\Pi$ is true.
We also use higher classes via the concept of \emph{precomputation on the parameter}.
\begin{definition}
	Let $\mathcal C$ be any complexity class.
	Then $\para\mathcal C$ is the class of all PPs $\Pi\subseteq\Sigma^*\times\mathbb N$ such that there exists a computable function $\pi\colon\mathbb N\to\Delta^*$ and a language $L\in\mathcal C$ with $L\subseteq\Sigma^*\times\Delta^*$ such that for all $(x,k)\in\Sigma^*\times\mathbb N$ we have that $(x,k)\in \Pi \Leftrightarrow (x,\pi(k))\in L$.
\end{definition}
Observe that $\para\Ptime=\FPT$ is true.
For a constant $c\in\mathbb N$ and a PP $\Pi\subseteq\Sigma^*\times\mathbb N$, the \emph{$c$-slice of $\Pi$}, written as $\Pi_c$, is defined as $\Pi_c:=\{\,(x,k)\in\Sigma^*\times\mathbb N\mid k=c\,\}$.
Observe that, in our setting, showing $\Pi\in\para\mathcal C$, it suffices to show $\Pi_c\in\mathcal C$ for every $c\in\mathbb{N}$.
Consider the following special subclasses of formulas:
$$\begin{array}{@{}r@{\,}c@{\,}l@{}}
	\Gamma_{0, d} & =&  \setdefinition{\ell_1\land\ldots\land \ell_c \mid \ell_1,\ldots, \ell_c \text{ are literals and }c\leq d },\\
	\Delta_{0, d} & = & \setdefinition{\ell_1\lor\ldots\lor \ell_c \mid \ell_1,\ldots, \ell_c \text{ are literals and }c\leq d },\\
	\Gamma_{t, d} & = & \left\{\,\bigwedge\limits_{i\in I} \alpha_i \,\middle|\, \alpha_i \in \Delta_{t-1,d}  \text{ for } i \in I\, \right\}, \\
	\Delta_{t, d} &= & \left\{\,\bigvee\limits_{i\in I} \alpha_i \,\middle|\, \alpha_i \in \Gamma_{t-1,d},  \text{ } i \in I\, \right\}.
\end{array}$$

The parameterized weighted satisfiability problem ($\WSAT$) for propositional formulas is defined as below.
\paraproblemdef{$\wsat{t}{d}$}{a $\Gamma_{t,d}$-formula $\alpha$ with $t, d \geq 1 $ and $k\in \mathbb{N}$}{$k$}{is there a satisfying assignment for $\alpha$ of weight $k$}

The classes of the $\complClFont{W}$-hierarchy can be defined in terms of these problems.
\begin{proposition}[\cite{DBLP:series/txcs/DowneyF13}]\label{theorem-wsat}
	The problem $\wsat{t}{d}$ is \W t-complete for each $t \geq 1$ and $d\geq 1$, under $\fptreduction$-reductions.
\end{proposition}
\paragraph*{Logic-based Argumentation}
All formulas in this paper are propositional formulas. 
We follow the notion of Creignou et~al.~\cite{DBLP:journals/tocl/CreignouE014}.
\begin{definition}[\cite{DBLP:journals/ai/BesnardH01}]
	Given a set of formulas $\Phi$ and a formula $\alpha$, one says that $(\Phi,\alpha)$ is an \emph{argument (for $\alpha$)} if (1) $\Phi$ is consistent, (2) $\Phi\models\alpha$, and (3) $\Phi$ is subset-minimal w.r.t.\ (2).
	In case of $\Phi\subseteq\Delta$, $\argument$ is an \emph{argument in $\Delta$}. 
	We call $\alpha$ the \emph{claim}, $\Phi$ the \emph{support} of the argument, and $\Delta$ the \emph{knowledge-base}.	
\end{definition}

In this paper, we consider three problems from the area of logic-based argumentation, namely $\ARG$, $\ARGcheck$, and $\ARGrel$.
\problemdef{$\ARG$}{a set of formulas $\Delta$ and a formula $\alpha$}{is there a set $\Phi\subseteq \Delta$ such that $\argument$ is an argument in $\Delta$}
\problemdef{$\ARGcheck$}{a set of formulas $\Phi$ and a formula $\alpha$}{is $\argument$ an argument}
\problemdef{$\ARGrel$}{a set of formulas $\Delta$, and formulas $\psi\in\Delta$ and $\alpha$}{is there a set $\Phi\subseteq \Delta$ with $\psi\in\Phi$ such that $\argument$ is an argument in $\Delta$}

Turning to the parameterized complexity perspective on the introduced problems, immediate parameters that we consider are $|\enc{\mathcal X}|$ (size of the encoding of $\mathcal X$), $|\mathcal X|$ (number of formulas in $\mathcal X$), $|\Vars{\mathcal X}|$ (number of variables in $\mathcal X$) for $\mathcal X\in\{\Delta,\Phi\}$, as well as $|\enc{\alpha}|$ and $|\Vars{\alpha}|$.
Regarding the parameterized versions of the problems from above, e.g., $\p\ARG(k)$, where $k$ is a parameter, then defines the version of $\ARG$ parameterized by $k$, accordingly.

In the following, we want to formally relate the mentioned notions of encoding length, number of variables, as well as number of formulas.
We will see that bounding the encoding length, implies having limited space for encoding variables and, in turn, restricts the number of possible formulas.
However, the converse is also true: if one bounds the number of variables, then one also has limited possibilities about defining \emph{different} formulas.
The following definition makes clear what `different' means in our context.
\begin{definition}[Formula redundancy]
	A CNF-formula $\varphi=\bigwedge_{i=1}^{m}C_i$, with $C_i=(\ell_{i,1}\lor\cdots\lor\ell_{i,n_i})$ is \emph{redundant} if there exist $1\leq i\neq j\leq m$ such that $\{\,\ell_{i,k}\mid 1\leq k\leq n_i\,\}=\{\,\ell_{j,k}\mid 1\leq k\leq n_j\,\}$.
\end{definition}
\begin{example}
	The formulas $x\land x$ and $(x\lor x\lor y)\land (x\lor y)$ are redundant. 
	The formulas $x\land y$ and $(x\lor y)\land x$ are not redundant.
\end{example}

Liberatore \cite{DBLP:journals/ai/Liberatore05} studied a stronger notion of redundancy in the context of CNF-formulas, namely, on the level of implied clauses.
We do not need such a strict notion of redundancy here, as the weaker notion suffices for proving the following Lemma. 
As a result, in the following, we consider only formulas that are just not redundant.
The redundancy (in our context) can be straightforwardly checked in time quadratic in the length of the given formula.
\begin{lemma}\label{lem:encoding-number-variables}
For any set of CNF-formulas $\Phi$, we have that
\begin{enumerate}
	\item $|\Phi|\leq2^{2^{2\cdot|\Vars{\Phi}|}}$,
	\item $f(|\Vars{\Phi}|)\leq |\enc{\Phi}|$, where $f$ is some computable function, and
	\item $|\enc{\Phi}|\leq{|\Phi|}^3$.
\end{enumerate}
\end{lemma}
\begin{proof}
\begin{enumerate}
	\item Let $v\in\mathbb{N}$ be a fixed number of variables.
	As we consider CNF-formulas, a formula consists of clauses of literals.
	The number of possible clauses then is the number of subsets of possible literals $\{x_1,\dots,x_v,\lnot x_1,\dots,\lnot x_v\}$, that is, $2^{2\cdot v}$-many.
	A CNF-formula is a subset of the set of possible clauses.
	As a result, we have $2^{2^{2\cdot v}}$-many possible CNF-formulas.
	\item We represent a variable $x_i$ by its binary encoding. 
	Clearly, $|\enc{\Phi}|=\sum_{\varphi\in\Phi}|\enc{\varphi}|$.
	However, 
	\begin{align*}
		|\enc{\varphi}|&\leq|\varphi|\cdot\log(|\Vars{\varphi}|)+|\varphi|\\
		&=|\varphi|\cdot(\log(|\Vars{\varphi}|)+1)\\
		&\leq |\Phi|\cdot(\log(|\Vars{\Phi}|)+1).
	\end{align*}
	As a result, we get 
	\begin{align*}
		|\enc{\Phi}| &\leq |\Phi|\cdot\max_{\varphi\in\Phi}|\enc{\varphi}|\\
		&\leq|\Phi|^2\cdot(\log(|\Vars{\Phi}|)+1)
		\intertext{As $|\Phi|\leq|\enc{\Phi}|$ is true, we have that} 
		|\enc{\Phi}| &\leq|\enc{\Phi}|^2\cdot(\log(|\Vars{\Phi}|)+1)\\
		\Leftrightarrow
		(\log(|\Vars{\Phi}|)+1)^{-1}&\leq|\enc{\Phi}|
	\end{align*}
	\item As we only consider formulas that are not redundant, the encoding length of a set of formulas contains information about the number of its formulas.
	We have that $|\enc{\Phi}| \leq|\Phi|^2\cdot(\log(|\Vars{\Phi}|)+1)$ (as in (2.)).
	However, $|\Phi|\leq2^{2^{2\cdot |\Vars{\Phi}|}}$, and, as a result, we get 
	\[
		|\enc{\Phi}| \leq|\Phi|^2\cdot(\log(\log(2\cdot|\Phi|))+1)\leq|\Phi|^3.\qedhere
	\]
\end{enumerate}

\end{proof}
 
Notice that due to Lemma~\ref{lem:encoding-number-variables}, the problems $\ARG$, $\ARGcheck$, $\ARGrel$ parameterized with respect to any of the parameters for the respective three (two) variants introduced above are $\FPT$-equivalent.
As a result, we will choose the one of the three (two) variants in our results that is technically most convenient.
Notice also that the parameter $|\Phi|$ only makes sense for $\ARGcheck$, whereas $|\Delta|$ makes sense only for the other two problems, that is, $\ARG$ and $\ARGrel$. 

\subsection{Schaefer's Framework}
For a deeper introduction into Schaefer's CSP framework, consider the article of Böhler~et~al.~\cite{bcrv04}.

A \emph{logical relation} of arity $k\in\mathbb N$ is a relation $R\subseteq\{0,1\}^k$, and a \emph{constraint} $C$ is a formula $C = R(x_1,\dots,x_k)$, where $R$ is a $k$-ary logical relation, and $x_1,\dots,x_k$ are (not necessarily distinct) variables.
If $V$ is a set of variables and $u$ a variable, then $C[V/u]$ denotes the constraint obtained from $C$ by replacing every occurrence of every variable of $V$ by $u$.
An assignment $\theta$ \emph{satisfies $C$}, if $(\theta(x_1),\dots,\theta(x_k))\in R$.
A \emph{constraint language} (CL) $\Gamma$ is a finite set of logical relations, and a \emph{$\Gamma$-formula} is a conjunction of constraints over elements from $\Gamma$.
Eventually, a $\Gamma$-formula $\varphi$ is \emph{satisfied} by an assignment $\theta$, if $\theta$ simultaneously satisfies all constraints in it.
In such a case $\theta$ is also called a \emph{model of $\varphi$}.
Whenever a $\csl$-formula or a constraint is logically equivalent to a single clause or term or literal, we treat it as such.
We say that a $k$-ary relation $R$ is represented by a formula $\phi$ in CNF if $\phi$ is a formula over $k$ distinct variables $x_1,\ldots,x_k$ and $\phi\equiv R(x_1,\ldots, x_k)$.
Moreover, we say that $R$ is 
\begin{itemize}
\item \emph{Horn} (resp., \emph{dual-Horn}) if $\phi$ contains at most one positive (negative) literal per each clause.
\item \emph{Bijunctive} if $\phi$ contains at most two literals per each clause.
\item \emph{Affine} if $\phi$ is a conjunction of linear equations of the form $x_1\oplus \ldots \oplus x_n=a$ where $a\in \{0,1\}$.
\item \emph{Essentially negative} if every clause in $\phi$ is either negative or unit positive. $R$ is \emph{essentially positive} if every clause in $\phi$ is either positive or unit negative.
\item \emph{$1$-valid} (resp., \emph{$0$-valid}) if every clause in $\phi$ contains at least one positive (negative) literal.
\end{itemize}
Furthermore, we say a relation is \emph{Schaefer} if it is Horn, dual-Horn, bijunctive, or affine.
We say that a relation is \emph{$\varepsilon$-valid} if it is $1$- or $0$-valid or both.
Finally, for a property \emph{P} of a relation, we say that a CL $\csl$ is \emph{P} if all relations in $\csl$ are \emph{P}.

\begin{definition}
	\begin{enumerate}
	\item The set $\clos{\csl}$ is the smallest set of relations that contains $\csl$, the equality constraint, $=$, and which is closed under primitive positive first order definitions, that is, if $\phi$ is an $\csl \cup \{=\}$-formula and $R(x_1, \dots, x_n) \equiv \exists y_1 \dots \exists y_l \phi(x_1, \dots, x_n,y_1, \dots, y_l)$, then $R \in \clos{\csl}$. In other words, $\clos{\csl}$ is the set of relations that can be expressed as a $\csl \cup \{=\}$-formula with existentially quantified variables.
	\item The set $\clos{\csl}_{\neq}$ is the set of relations that can be expressed as a $\csl$-formula with existentially quantified variables (no equality relation is allowed).
	\item The set $\clos{\csl}_{\not\exists,\neq}$ is the set of relations that can be expressed as a $\csl$-formula  (neither the equality relation nor existentially quantified variables are allowed).
	\end{enumerate}
\end{definition}
	
The set $\clos{\csl}$ is called a \emph{relational clone} or a \emph{co-clone} with \emph{base} $\csl$ \cite{DBLP:journals/ipl/BohlerRSV05}. 
Notice that for a co-clone $\cocl$ and a CL $\csl$ the statements $\csl \subseteq \cocl$, $\clos{\csl} \subseteq \cocl$, $\closneq{\csl} \subseteq \cocl$ and $\closnexneq{\csl} \subseteq \cocl$ are equivalent. 
Throughout the paper, we refer to different types of Boolean relations and corresponding co-clones following Schaefer's terminology \cite{DBLP:conf/stoc/Schaefer78}.
For a tabular overview of co-clones, relational properties, and bases, we refer the reader to  Table~\ref{tab:bases}.
Note that $\closneq{\csl} \subseteq \clos{\csl}$ is true by definition. 
The other direction is not true in general. 
However, if $(x = y) \in \closneq{\csl}$, then we have that $\closneq{\csl} = \clos{\csl}$.
\begin{example}

\begin{table*}
	\centering
	\includegraphics[width=\linewidth]{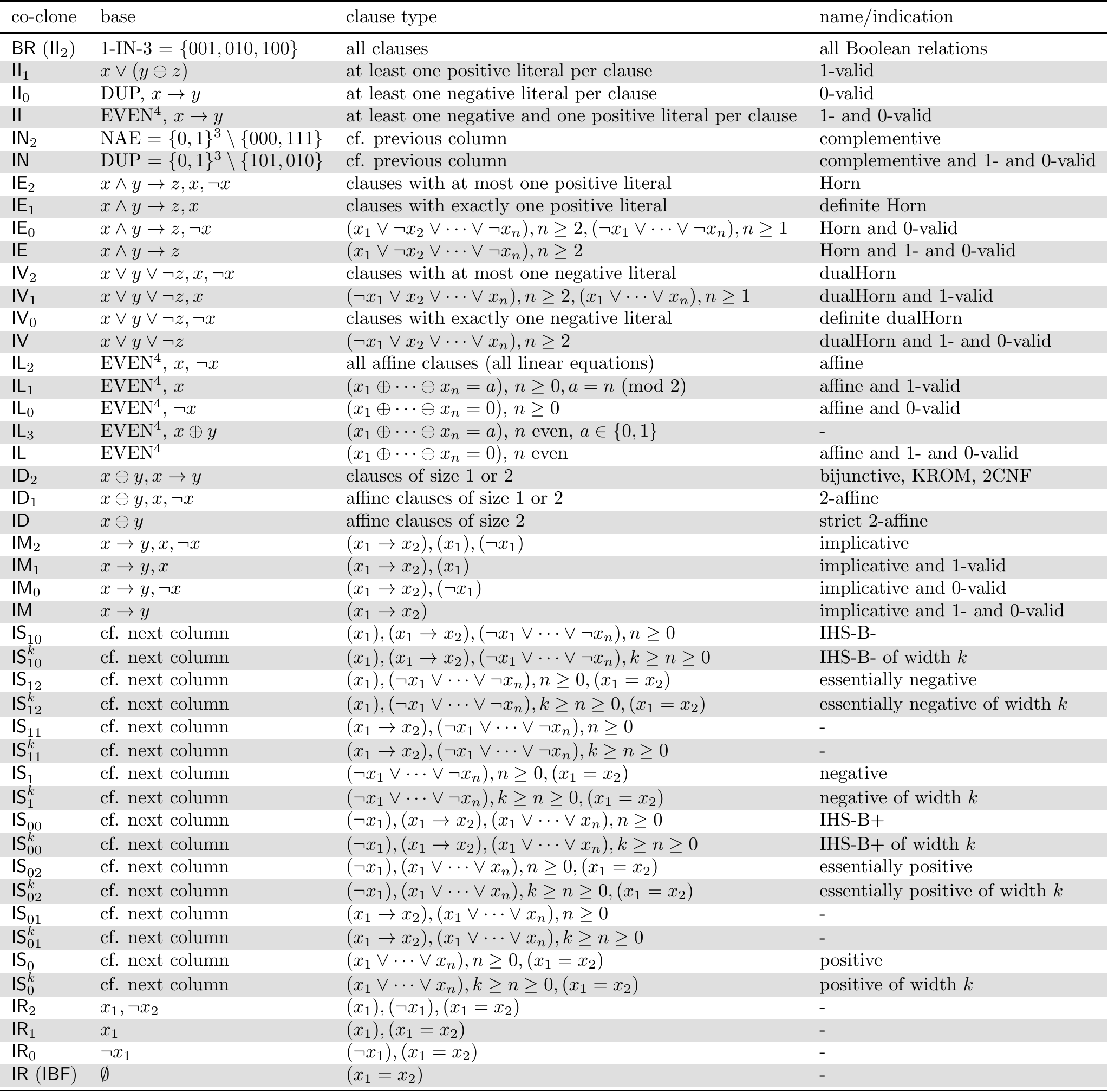}
	\caption{Overview of bases \cite{DBLP:journals/ipl/BohlerRSV05} and clause descriptions \cite{DBLP:journals/ai/NordhZ08} for co-clones, where EVEN$^4$ = $x_1 \oplus x_2 \oplus x_3 \oplus x_4 \oplus 1$.}\label{tab:bases}
\end{table*}

Let $R(x_1,x_2,x_3)\dfn (x_1\lor x_2\lor x_3)\land (\neg x_1\lor \neg x_2\lor \neg x_3)$. 
Then 

$${(x_1\lor x_2)\land (x_2 \oplus x_3 = 0)} \equiv \exists y (R(x_1,x_2,y) \land F(y) \land (x_2=x_3)),$$ 
where $F=\{0\}$. This implies that $(x_1\lor x_2)\land (x_2 \oplus x_3 = 0)\in \clos{\setdefinition{R,F}}$.
\end{example}
\subsection{Technical Implementation Results}
We say a Boolean relation $R$ is \emph{strictly essentially positive} (resp., \emph{strictly essentially negative}) if it can be defined by a conjunction of literals and positive clauses (resp., negative clauses) only. Note that the only difference to essentially positive (resp., essentially negative) is the absence of the equality relation (see Table~\ref{tab:bases}).
We abbreviate in the following essentially positive by ``ess.pos.'' and essentially negative by ``ess.neg.''.

\begin{proposition}{\cite[Lem.~7]{DBLP:conf/lfcs/0002M020}}]\label{lem:impl-essPosNeg_old}
Let $\Gamma$ be a CL that is neither ess.pos., nor ess.neg. Then, we have that $(x=y) \in \closneq{\Gamma}$ and $\clos{\Gamma} = \closneq{\Gamma}$.
\end{proposition}

With the following implementation result we can strengthen this statement to Lemma~\ref{lem:equality_available}.

\begin{lemma}\label{lem:impl-essPosNeg}
Let $\Gamma$ be a CL that is not $\varepsilon$-valid.
If $\Gamma$ is ess.neg.\ and not strictly ess.neg.\, or ess.pos.\ and not strictly ess.pos.\, then we have that $(x = y) \land t \land \neg f \in \closnexneq{\Gamma}$.
\end{lemma}
\begin{proof}
We prove the statement for $\Gamma$ that is ess.pos.\ but not strictly ess.pos. The other case can be treated analogously. W.l.o.g., let $\Gamma = \{R\}$, thus $R$ is ess.pos.\ but not strictly ess.pos. 
Furthermore, $R$ is neither 1-valid nor 0-valid.  
Let $R$ be of arity $k$ and let $V = \{x_1, \dots, x_k\}$ be a set of $k$ distinct variables. 
By definition of ess.pos. (cf. $\IS{}{02}$ in Table~\ref{tab:bases}), $R$ can be written as conjunction of negative literals, positive clauses and equalities. 

If $R$ can be written without any equality, then $R$ is strictly ess.pos., a contradiction. 
As a result, any representation of $R$ as conjunction of negative literals, positive clauses and equalities requires at least one equality. 
Suppose, w.l.o.g., that $R(x_1, \dots, x_k) \models (x_1 = x_2)$, while $R(x_1,\dots,x_k) \not\models x_1$ and $R(x_1,\dots,x_k) \not\models \neg x_1$.
We define the following three subsets of $V$:
$W = \{\,x_i \mid R(x_1,\dots, x_k) \models (x_1 = x_i)\,\}$, 
$N = \{\,x_i \mid R(x_1,\dots, x_k) \models \neg x_i\,\}$, and 
$P = V \setminus (W \cup N)$

By construction the three sets provide a partition of $V$. 
Then, $W$ is nonempty by construction, $N$ is nonempty since $R$ is not 1-valid and $P$ is nonempty since $R$ is not 0-valid.
Denote by $C$ the $\{R\}$-constraint $C = R(x_1, \dots, x_k)$. 
Consider the constraint $M(x_1, x_2, t, f ) = C[W/x_2, P/t, N/f]$. 
One verifies that $M(x_1, x_2, t, f ) \equiv (x_1 = x_2) \land t \land \neg f$.
\end{proof}

\begin{lemma}\label{lem:equality_available}
Let $\Gamma$ be a CL that is neither strictly ess.pos., nor strictly ess.neg.
Then $(x=y) \in \closneq{\Gamma}$ and $\clos{\Gamma} = \closneq{\Gamma}$.
\end{lemma}
\begin{proof}
If $\Gamma$ is not ess.pos. and not ess.neg.\ the statement follows from Proposition~\ref{lem:impl-essPosNeg_old}. 
Note that this lemma's statement implies that $\Gamma$ is not $\varepsilon$-valid. 
If $\Gamma$ is ess.pos.\ or ess.neg., by Lemma~\ref{lem:impl-essPosNeg} we have $(x = y) \land t \land \neg f \in \closnexneq{\Gamma}$. 
Conclude by noticing that
$(x = y) \equiv \exists t\exists f\, (x = y) \land t \land \neg f \in \closneq{\Gamma}$.
\end{proof}

\begin{lemma}\label{lem:impl-epsilon-valid}
Let $\Gamma$ be a CL that is neither $\varepsilon$-valid, nor ess.pos., nor ess.neg. Then, if $\Gamma$ is 
\begin{enumerate}
\item not Horn, not dual-Horn, and not complementive, then $(x\neq y) \land t \land \neg f \in \closnexneq{\Gamma}$,\label{label-lem-impl-epsilon-valid-nothorn-notdualhorn-notcomplementive}
\item not Horn, not dual-Horn, and complementive, then $(x\neq y) \in \closnexneq{\Gamma}$, and\label{label-lem-impl-epsilon-valid-nothorn-notdualhorn-complementive}
\item Horn or dual-Horn, then $(x=y) \land t \land \neg f \in \closnexneq{\Gamma}$.\label{label-lem-impl-epsilon-valid-horn-or-dualhorn}
\end{enumerate}
\end{lemma}
\begin{proof}
This follows immediately from the proof of Proposition~\ref{lem:impl-essPosNeg_old}. The proof given in \cite[Lemma~7]{DBLP:conf/lfcs/0002M020, DBLP:journals/corr/abs-1906-00703} makes a case distinction according to whether $\csl$ is 0-valid and/or 1-valid. In the case of non $\varepsilon$-valid $\csl$ a further case distinction is made according to whether $\csl$ is Horn and/or dualHorn. Here the statements 1., 2., and 3. are proven.
\end{proof}

Let us denote by $\T/\F$ the unary relations that implement true/false.
That is, $\T=\{(1)\}$ and $\F=\{(0)\}$. 
The following implementation results are folklore. 

\begin{proposition}[Creignou~et~al.~\cite{DBLP:books/daglib/0004131}]\label{lem:impl-stuff}
If $\Gamma$ is a CL that is
\begin{enumerate}
\item complementive and not $\varepsilon$-valid, then $(x\neq y)\in \closnexneq{\Gamma}$,\label{lem-label-impl-stuff-compl-notepsvalid}
\item neither complementive, nor $\varepsilon$-valid, then $(t \land \bar f) \in \closnexneq{\Gamma}$.\label{lem-label-impl-stuff-notcompl-not-eps-valid}
\item $1$-valid and not $0$-valid, then $\T \in \closnexneq{\Gamma}$,\label{lem-label-impl-stuff-1valid-not0valid}
\item $0$-valid and not $1$-valid, then $\F \in \closnexneq{\Gamma}$, and\label{lem-label-impl-stuff-0valid-not1valid}
\item $0$-valid and $1$-valid, then $(x=y)\in \closnexneq{\Gamma}$.\label{lem-label-impl-stuff-1valid-0valid}
\end{enumerate}
\end{proposition}

\subsection{Parameterized Implication Problem}
In this subsection, we consider the parameterized complexity of the implication problem ($\IMP$).
\problemdef{$\IMP(\Gamma)$}{a set of $\Gamma$-formulas $\Phi$ and a $\Gamma$-formula $\alpha$}{is $\Phi\models\alpha$ true}
For $\pIMP$, the parameterized version of $\IMP$, we consider the parameter $k\in\{|\Phi|,|\alpha|\}$, and also write $\pIMP(\Gamma,k)$.
The following corollary is due to Schnoor and Schnoor {\cite[Theorem~6.5]{DBLP:conf/dagstuhl/SchnoorS08}}.
They study a restriction of our problem $\IMP$, where $|\Phi|=1$.
\begin{corollary}
	Let $\Gamma$ be a CL.
	$\IMP(\Gamma)$ is in $\P$ when $\Gamma$ is Schaefer and $\co\NP$-complete otherwise.
\end{corollary}
Consequently, the parameterized problem $\p\IMP(\Gamma, k)$ is $\FPT$ when $\Gamma$ is Schaefer and $k\in \{|\Phi|, |\alpha|\}$.
We consider the cases when $\Gamma$ is not Schaefer.
In the following, we differentiate the restrictions on $\Phi$ from the ones on $\alpha$.
That is, we introduce a technical variant, $\IMP(\Gamma',\Gamma)$ of the implication problem. 
An instance of $\IMP{(\Gamma',\Gamma)}$ is a tuple $(\Phi,\alpha)$, where $\Phi$ is a set of $\Gamma'$-formulas and $\alpha$ is a $\Gamma$-formula.
The following corollary also follows from the work of Schnoor and Schnoor {\cite[Theorem~{6.5}]{DBLP:conf/dagstuhl/SchnoorS08}}.
\begin{corollary}\label{cor:IMP-variant}
	Let $\Gamma$ and $\Gamma'$ be non-Schaefer CLs. 
	If $\Gamma'\subseteq \clos{\Gamma}$ then $\IMP(\Gamma',\Gamma)\preduction \IMP(\Gamma)$.
\end{corollary}

Regarding non-Schaefer CLs, it turns out that the parameter $\alpha$ does not make the problem any easier. 
One possible explanation for this hardness is that the formulas in $\Phi$ and $\alpha$ do not necessarily share a set of variables. 

	


	



\begin{lemma}\label{lem:imp-alpha}
	The problem $\pIMP(\Gamma, |\alpha|)$ is $\para\co\NP$-complete when the CL $\Gamma$ is not Schaefer.
\end{lemma}

\begin{proof}
	Membership follows because the classical problem is in $\co\NP$.
	To achieve the lower bound, we reduce from the unsatisfiability problem. 
	That is, given a formula $\Phi$, the question is whether $\Phi$ is unsatisfiable.
	Moreover, checking unsatisfiability is $\co\NP$-complete for non-Schaefer languages (follows by Schaefer's~\cite{DBLP:conf/stoc/Schaefer78} $\SAT$ classification).

	We will inherently use Corollary~\ref{cor:IMP-variant} and make a case distinction as whether $(\Phi,\alpha)$ is $1$-valid, $0$-valid or complementive.
	\begin{description}
		\item[Case 1.] Let $\Gamma$ be $1$-valid and not $0$-valid. 
		We prove that for some well chosen $1$-valid language $\Gamma'$ and a $\Gamma'$-formula $\Phi$, the problem $\pIMP(\Gamma',\Gamma,|\alpha|)$ is $\para\co\NP$-hard. 
		According to item (\ref{lem-label-impl-stuff-1valid-not0valid}.) in Proposition~\ref{lem:impl-stuff}, $\T\in \closnexneq{\Gamma'}$.
		Let $\alpha= \T(x)$ and $\Phi$ be a $\Gamma'$-formula where $x$ does not appear.
		Then $\Phi\models \alpha$ if and only if $\Phi$ is unsatisfiable.
		This is because, if $\Phi$ is satisfiable then there is an assignment $s$ such that $s \models \psi$.
		This gives a contradiction because the assignment $s'$ that extends $s$ by $s'(x)=0$ satisfies that $s'\models \Phi$ and $s'\not\models \alpha$.
		
		\item[Case 2.] Let $\Gamma$ be $0$-valid and not $1$-valid. 
		According to item (\ref{lem-label-impl-stuff-0valid-not1valid}.) in Proposition~\ref{lem:impl-stuff}, $\F\in \closnexneq{\Gamma'}$.
		This case is similar to Case 1, as we take $\alpha=\F(x)$ and $\Phi$ a $\Gamma'$-formula not containing $x$.

		\item[Case 3.] Let $\Gamma$ be complementive but not $\varepsilon$-valid. 
		We prove that for some well chosen complementive language $\Gamma'$ and a $\Gamma$-formula $\alpha$, the problem $\pIMP(\Gamma',\Gamma,|\alpha|)$ is $\para\co\NP$-hard. 
		According to item (\ref{lem-label-impl-stuff-compl-notepsvalid}.) in Proposition~\ref{lem:impl-stuff}, $x\not = y \in \closnexneq{\Gamma'}$.
		Then, $\co\NP$-hardness follows, as for any set of $\Gamma'$-formulas $\Phi$, $\Phi \models (x\neq x)$ if and only if $\Phi$ is unsatisfiable.
	
		\item[Case 4.] Let $\Gamma$ be $0$- and $1$-valid.
		By Lemma~\ref{lem:impl-epsilon-valid}~(\ref{label-lem-impl-epsilon-valid-nothorn-notdualhorn-notcomplementive}.)/(\ref{label-lem-impl-epsilon-valid-nothorn-notdualhorn-complementive}.), we have access to `$\neq$'.
		We can state a reduction from the complement of $\SAT$ to $\pIMP(\Gamma,|\alpha|)$ as in Case 3.
		That is, $\Phi$ is unsatisfiable if and only if $\Phi\models x\neq x$ for a fresh variable $x$.\qedhere
	\end{description}
\end{proof}

Note that regarding the parameter $|\Phi|$, the problem $\pIMP(\Gamma,|\Phi|)$ is $\FPT$ if $\Gamma$ is Schaefer.
Otherwise, only $\co\NP$-membership is clear.

\section{Parameter: Size of the Claim $\alpha$}
In this section we discuss the complexity results regarding the parameter $\alpha$, that is, the number of variables and the encoding size of $\alpha$.
It turns out that the computational complexity of the argumentation problems is hidden in the structure of the underlying CL. 
That is, in many cases, considering the claim size as a parameter does not lower the complexity. 
This is proved by noting that certain slices of the parameterized problems already yield hardness results.

\begin{theorem}\label{theorem:arg-alpha}
$\p\ARG(\Gamma, |\alpha|)$, for a CL $\Gamma$, is
	\begin{enumerate}
	\item $\FPT$ if $\Gamma$ is Schaefer and $\varepsilon$-valid,\label{label-thm-arg-alpha-fpt}
    \item $\para\NP$-complete if $\Gamma$ is Schaefer and neither $\varepsilon$-valid, nor strictly ess.pos., nor strictly ess.neg.,\label{label-thm-arg-alpha-paraNP}
	\item in $\W{1}$ if $\Gamma$ is strictly ess.neg.\ and strictly ess.pos.,\label{label-thm-arg-alpha-w1}
	\item in $\W{2}$ if $\Gamma$ is strictly ess.neg.\ or strictly ess.pos.,\label{label-thm-arg-alpha-w2}
	\item $\para\co\NP$-complete if $\Gamma$ is not Schaefer and $\varepsilon$-valid, and\label{label-thm-arg-alpha-paraconp}
	\item $\para\SigmaP$-complete if $\Gamma$ is not Schaefer and not $\varepsilon$-valid.\label{label-thm-arg-alpha-parasigmap2}
\end{enumerate}
\end{theorem}
\begin{proof}
(\ref{label-thm-arg-alpha-fpt}.) The classical problem $\ARG(\Gamma)$ is already in $\P$ for this case \cite[Thm~5.3]{DBLP:journals/tocl/CreignouE014}.
(\ref{label-thm-arg-alpha-paraNP}.) The upper bound follows because the unparameterized problem $\ARG(\Gamma)$ is in $\NP$ \cite[Prop~5.1]{DBLP:journals/tocl/CreignouE014}. 
The lower bound is proven in Lemmas~\ref{lem:argParaNP-T.F.EQ}, \ref{lem:argParaNP-EQ-NEQ} and \ref{lem:argParaNP-essPosNeg}.
(\ref{label-thm-arg-alpha-w1}.) is proven in Lemma~\ref{lemW1}.
(\ref{label-thm-arg-alpha-w2}.) is proven in Lemma~\ref{lemW2}.

For (\ref{label-thm-arg-alpha-paraconp}.) (resp., (\ref{label-thm-arg-alpha-parasigmap2}.)), the membership follows because the classical problem is in $\co\NP$ (resp., $\SigmaP$) \cite[Thm~5.3]{DBLP:journals/tocl/CreignouE014}. 
For hardness of $\p\ARG(\Gamma,|\alpha|)$ when $\Gamma$ is $\varepsilon$-valid, notice that, since $\Delta$ is $\varepsilon$-valid, an instance $(\Delta,\alpha)$ of $\p\ARG$ admits an argument if and only if $\Delta\models \alpha$.
The result follows from Lemma~\ref{lem:imp-alpha} because the implication problem is still $\para\co\NP$-hard.
Finally, when $\Gamma$ is not Schaefer and not $\varepsilon$-valid, in the proofs of Creignou~et~al.~\cite[Prop.~5.2]{DBLP:journals/tocl/CreignouE014} the constructed reductions define $\alpha$ whose length is 2 or 3.
Accordingly, either the 2-slice or the 3-slice is $\SigmaP$-hard.
This gives the desired hardness result.  
\end{proof}

For technical reasons we introduce the following variant of the argumentation existence problem.
\problemdef{$\ARG(\Gamma, R)$}{a set of $\Gamma$-formulas $\Delta$ and an  $R$-formula $\alpha$}{$\exists$ $\Phi\subseteq \Delta$ s.t.\ $\argument$ is an argument in $\Delta$}

\begin{lemma}\label{lem:technical-variants}
Let $\Gamma, \Gamma'$ be two CLs and $R$ a Boolean relation. 
If $\Gamma' \subseteq \closneq{\Gamma}$ and $R\in\closnexneq{\Gamma}$, then $\ARG(\Gamma', R) \leqlogm \ARG(\Gamma)$.
\end{lemma}
\begin{proof}
Let $(\Delta, \alpha)$ be an instance of the first problem, where $\Delta = \{\,\delta_i \mid i \in I\,\}$ and $\alpha = R(x_1, \dots, x_k)$. 
We map this instance to $(\Delta', \alpha')$, where $\Delta' = \{\delta'_i \mid \delta_i \in \Delta\}$ and $\alpha'$ is a $\Gamma$-formula equivalent to $R(x_1, \dots, x_k)$ (which exists because $R\in\closnexneq{\Gamma}$). 
For $i \in I$ we obtain $\delta'_i$ from $\delta_i$ by replacing $\delta_i$ by an equivalent $\Gamma$-formula with existential quantifiers (such a representation exists since $\Gamma' \subseteq \closneq{\Gamma}$) and deleting all existential quantifiers.
\end{proof}
Note that the previous result is only used to show lower bounds for specific slices and, accordingly, is stated in the classical setting.

\begin{lemma}\label{lem:argParaNP-T.F.EQ}
If the CL $\Gamma$ is neither affine, nor $\varepsilon$-valid, nor ess.pos., nor ess.neg., then $\p\ARG(\Gamma, |\alpha|)$ is $\para\NP$-hard.
\end{lemma}
\begin{proof}
We give a reduction from the $\NP$-complete problem $\PosOneThreeSAT$ such that $|\alpha|$ is constant. An instance of $\PosOneThreeSAT$ is a $\threeCNF$-formula with only positive literals, the question is to determine whether there is a satisfying assignment which maps exactly one variable in each clause to true. 
We make a case distinction according to the case (\ref{label-lem-impl-epsilon-valid-nothorn-notdualhorn-notcomplementive}.) and (\ref{label-lem-impl-epsilon-valid-horn-or-dualhorn}.) in Lemma~\ref{lem:impl-epsilon-valid}. 
Case (\ref{label-lem-impl-epsilon-valid-nothorn-notdualhorn-complementive}.) is not needed as if $\Gamma$ is not affine, not horn and not dual-Horn, 
then $\Gamma$ can not be complementive.
We first treat case (\ref{label-lem-impl-epsilon-valid-horn-or-dualhorn}.), that is, we have that $(x=y) \land t \land \neg f \in \closnexneq{\Gamma}$. We then show that the other two cases can be treated with minor modifications of the procedure.

Let $\varphi$ be an instance of $\PosOneThreeSAT$.
We first reduce $\varphi$ to and instance $(\Delta,\alpha)$ of $\ARG(\{\T,\F,=\}, (x=y) \land t \land \neg f)$, and then conclude with Lemmas~\ref{lem:impl-epsilon-valid} and \ref{lem:technical-variants}.
Given $\varphi = \bigwedge_{i=1}^k (x_i \lor y_i \lor z_i)$, an instance of $\PosOneThreeSAT$ and let $t,f,c_1, \dots, c_{k+1}$ be fresh variables.
We let $\Delta$ and $\alpha$ as following.
$$\begin{array}{r@{\,}l}
\Delta  =\; & \bigcup_{i=1}^k\{ x_i \land \neg{y_i} \land \neg{z_i}  \land (c_i = c_{i+1})\land t \land \neg f\} \\
     \cup\;	& \bigcup_{i=1}^k\{ \neg{x_i} \land y_i \land \neg{z_i} \land (c_i = c_{i+1})\land t \land \neg f\} \\
		 \cup\; & \bigcup_{i=1}^k\{ \neg{x_i} \land \neg{y_i} \land z_i \land (c_i = c_{i+1})\land t \land \neg f\}, \\
\alpha =\;& (c_1 = c_{k+1}) \land t \land \neg f.
\end{array}$$
Note that any formula in $\Delta$ is expressible as $\Gamma$-formula since $\{\T,\F,=\} \subseteq \IM_2 \subseteq \clos{\Gamma}$ (cf.~\cite[Table~1]{DBLP:conf/lfcs/0002M020}). 
Since by Lemma~\ref{lem:equality_available}, $\closneq{\Gamma} = \clos{\Gamma}$ and  by construction $(x=y) \land t \land \neg f \in \closnexneq{\Gamma}$, we have, by Lemma~\ref{lem:technical-variants}, the desired reduction to $\p\ARG(\Gamma,|\alpha|)$. 
Note that in the reduction of Lemma~\ref{lem:technical-variants} the size of $\alpha$ is always constant.

For case (\ref{label-lem-impl-epsilon-valid-nothorn-notdualhorn-notcomplementive}.) of Lemma~\ref{lem:impl-epsilon-valid} we have that $(x\neq y) \land t \land \neg f \in \closnexneq{\Gamma}$. To cope with this change in the reduction we introduce one additional variable $d$ and replace $\alpha$ by $(c_1 \neq d) \land (d \neq c_{k+1}) \land t \land \neg f$.
\end{proof}

\begin{lemma}\label{lem:argParaNP-EQ-NEQ}
If the CL $\Gamma$ is affine, neither $\varepsilon$-valid, nor ess. pos., nor ess.neg., then $\p\ARG(\Gamma, |\alpha|)$ is $\para\NP$-hard.
\end{lemma}
\begin{proof}
We proceed analogously to the proof of Lemma~\ref{lem:argParaNP-T.F.EQ}.
We give a reduction from the $\NP$-complete problem $\PosOneThreeSAT$ such that $|\alpha|$ is constant. 
We make a case distinction according to case (\ref{label-lem-impl-epsilon-valid-nothorn-notdualhorn-notcomplementive}.) and (\ref{label-lem-impl-epsilon-valid-nothorn-notdualhorn-complementive}.) in Lemma~\ref{lem:impl-epsilon-valid} (case 3. can not occur for $\Gamma$ is affine and not ess.pos.). 
First, we treat the second case, that is, we have that $(x\neq y) \in \closnexneq{\Gamma}$. 
Then, we show that the first case can be treated with minor modifications of the procedure.

Now, we reduce $\PosOneThreeSAT$ to $\ARG(\{=,\neq\}, \{\neq\})$, and then conclude with Lemmas~\ref{lem:impl-epsilon-valid} and \ref{lem:technical-variants}.
We give the following reduction.
Let $\varphi = \bigwedge_{i=1}^k (x_i \lor y_i \lor z_i)$ be an instance of $\PosOneThreeSAT$ and let $t,d,c_1, \dots, c_{k+1}$ be fresh variables.
We map $\varphi$ to $(\Delta, \alpha)$, where
$$\begin{array}{@{}r@{\,}l@{}}
\Delta = &\bigcup_{i=1}^k\{ (x_i = t) \land (y_i \neq t) \land (z_i \neq t)  \land (c_i = c_{i+1})\} \\
  \cup   &\bigcup_{i=1}^k\{ (x_i\neq t) \land (y_i=t) \land (z_i\neq t) \land (c_i = c_{i+1})\} \\
	\cup	 & \bigcup_{i=1}^k\{ (x_i\neq t) \land (y_i\neq t) \land (z_i=t) \land (c_i = c_{i+1})\}, \\
\alpha =& (c_1 \neq d) \land (d \neq c_{k+1}).
\end{array}$$
Note that any formula in $\Delta$ is expressible as $\Gamma$-formula since $\{=,\neq\} \subseteq \ID \subseteq \clos{\Gamma}$ (cf.~\cite[Table~1]{DBLP:conf/lfcs/0002M020}). 
Since by Lemma~\ref{lem:equality_available} $\closneq{\Gamma} = \clos{\Gamma}$ and  by construction $(x\neq y) \in \closnexneq{\Gamma}$, we have by Lemma~\ref{lem:technical-variants} the desired reduction to $\p\ARG(\Gamma)$. Note that in the reduction of Lemma~\ref{lem:technical-variants} the size of $\alpha$ is always constant.

For case (\ref{label-lem-impl-epsilon-valid-nothorn-notdualhorn-notcomplementive}.) of Lemma~\ref{lem:impl-epsilon-valid} we have that $(x\neq y) \land t \land \neg f \in \closnexneq{\Gamma}$. 
To cope with this change in the reduction, we introduce one additional variable $f$ and add the constraints $t \land \neg f$ to $\alpha$ as well as to every formula in $\Delta$ .
\end{proof}

\begin{lemma}\label{lem:argParaNP-essPosNeg}
Let $\Gamma$ be a CL that is not $\varepsilon$-valid. If $\Gamma$ is ess.pos.\ and not strictly ess.pos.\ or ess.neg.\ and not strictly ess.neg., then $\p\ARG(\Gamma, |\alpha|)$ is $\para\NP$-hard.
\end{lemma}
\begin{proof}
We can use exactly the same reduction as in Lemma~\ref{lem:argParaNP-T.F.EQ}, except we do not require a case distinction.
Note that, by Proposition~\ref{lem:impl-stuff}, we have that $(t \land \neg f) \in \closnexneq{\Gamma}$.
Since $\exists f\, (t\land \neg f) \equiv \T(t)$ and $\exists t\, (t\land \neg f) \equiv \F(f)$, we conclude that $\T,\F \in \closneq{\Gamma}$. 
Further, by Lemma~\ref{lem:equality_available}, we have that $(x = y) \in \closneq{\Gamma}$. 
Together we have $\{\T,\F,= \} \subseteq \closneq{\Gamma}$, and thus any formula in $\Delta$ is expressible as $\Gamma$-formula with existential quantifiers but without equality.
By Lemma~\ref{lem:impl-essPosNeg}, it follows that $(x=y) \land t \land \neg f \in \closnexneq{\Gamma}$. Hence we can apply Lemma~\ref{lem:technical-variants} to conclude.
\end{proof}

\begin{lemma}\label{lemW1}
Let $\Gamma$ be a CL that is strictly ess.neg.\ and strictly ess.pos., then $\p\ARG(\Gamma, |\alpha|) \in \W{1}$.
\end{lemma}
\begin{proof}
	We give a reduction to the $\W1$-complete problem $\clique$.
	Note that by definition of strictly ess.pos. and strictly ess.neg., any $\Gamma$-formula can be written as a $\{\T, \F\}$-formula. 
	Let $(\Delta,\alpha)$ be an instance of $\ARG(\{\T,\F\}, |\alpha|)$. 
	Note that $\Delta$ is a set of terms and $\alpha$ is also a term.
	Let $\Delta = \{t_1, \dots, t_n\}$ and $\alpha = l_1 \land \dots \land l_k$. 
	Then we have the following two observations.
	\begin{description}
		 \item[Observation 1:] There is a support for $\alpha$ iff there is a support of cardinality at most $k$: If $l_i$ can be explained at all, then one
		$\varphi \in \Delta$ is sufficient.
		In other words, there is no $l_i$ in $ \alpha$ such that a combination of two terms from $\Delta$ is necessary to explain $l_i$.
		\item[Observation 2:] If a set of terms is pairwise consistent, then the whole set is consistent. If a set of terms is inconsistent, then there are two terms which are pairwise inconsistent.
	\end{description}

	For each literal $l_i\in \alpha$, form the sets $L^+_i=\setdefinition{t\in \Delta \mid l_i\in t}$ and $L^-_i=\setdefinition{t\in \Delta \mid \neg l_i\in t}$.
	That is, each $t\in L^+_i$ is a candidate support, whereas, no $t\in L^-_i$ can be in the support, for every $i\leq k$.
	Let $N=\bigcup\limits_{i\leq k} L^-_i$ and denote $L_i = L^+_i\backslash N$.
	It is important to notice that there is a support only if $L_i\not = \emptyset$ for each $i\leq k$.
	Otherwise, for some $i$, the support $\Phi$ can not contain a term $t$ supporting $l_i$ such that $\Phi$ is consistent.
	It remains to determine whether $\Phi$, that includes one $t$ from each $L_i$ is consistent. 
	The consistency still needs to be checked because terms in $\Delta$ may contain literals not in $\alpha$.
	That is, care should be taken when selecting which terms to include in the support.
	
	Consider the following graph $\mathcal G =(V,E)$.
	There is one node corresponding to each element $t_j$ and each set $L_i$.
	By slightly abusing the notation, we write $V=\bigcup \limits_{i\leq k}L_i$. 
	It is worth mentioning that if a term appears in two different sets, say $L_r$ and $L_s$, then there are distinct nodes for each term. 
	We will explain later why this is required.
	Finally, the edge relation denotes the pairwise consistency of terms.
	That is, there is an edge between a term $t_{r,i} \in L_i$ and $t_{s,j}\in L_j$ if $t_{r,i}\land t_{s,j}$ is consistent.
	However, there is no edge between $t_{r, i}$ and $t_{s, i}$.
	That is, if the two terms belong to the same set $L_i$, or if two terms belong to two different sets $L_i$ and $L_j$ but these are pairwise inconsistent, then there is no edge.
	
	We first prove that the reduction is indeed $\FPT$.
	The sets $L^+_i$ and $L^-_i$ can be computed in polynomial time. 
	The size of $V$ is $O(k\cdot n^2)$, because there are $k$ sets of the form $L_i$, each contains at most $n$ terms of size at most $n$ (where $n$ is the input size).
	To draw the edges, for each element $t\in L_1$, one needs to check the pairwise consistency for each of the remaining $k-1$ sets, each of size $O(n^2)$.
	This gives  $O(k\cdot n^3)$ time for one element of $L_1$.
	To determine edges for each element of $L_1$, it requires $n\cdot O(k\cdot n^3) = O(k\cdot n^4)$ time.
	Finally, to repeat this for each set $L_i$, it requires
	$O(k^2\cdot n^4)$ time.
	This proves that the reduction can be preformed in $\FPT$-time.
	
	Now we prove that the reduction preserves the cliques of $\mathcal G$ of size $k$ and the supports of $(\Delta,\alpha)$. 
	\begin{claim}\label{lemW1-claim}
		$(V,E)$ admits a clique of size $k$ if and only if $(\Delta,\alpha)$ admits a support.
	\end{claim}

\begin{proof}[Proof of Claim~\ref{lemW1-claim}]
	``$\Rightarrow$''. Let $S\subseteq V$ be a clique of size $k$ in $(V,E)$. 
	Since there are no edges between the two elements from the same set $L_i$, this implies $S$ contains exactly one term from each $L_i$.
	Furthermore, the fact that $S$ is a clique implies that the set of terms is consistent. 
	This provides a support for $\alpha$.
	
	``$\Leftarrow$''. Let $\Phi$ be a support for $\alpha$ in $\Delta$.
	According to observation 1, for each $l_i\in \alpha$, there is a term $t\in \Delta$ such that $t\models l_i$. 
	The set $L_i$ contains every such a $t\in \Delta$.
	Moreover, this holds for each $l_i\in \alpha$, this implies that $\Phi$ contains at least one $t$ from $L_i$ for each $i\leq k$.
	Finally, since $\Phi$ is consistent, this implies that every pair of nodes corresponding to the terms in $\Phi$ contains an edge.
	This gives a clique in $(V,E)$.
\end{proof}
	It might happen that there is one term $t\in \Delta$ such that $t\models l_i \land l_j$, and due to our construction, $t\in L_i \cap L_j$.
	However, the graph contains two separate nodes for each occurrence of $t$.
	This is required to ensure that a support of size smaller than $k$ also guarantees  a clique of size $k$ for $(V,E)$.
	\end{proof}

\begin{lemma}\label{lemW2}
Let $\Gamma$ be a CL that is strictly ess.neg. or strictly ess.pos.,  then $\p\ARG(\Gamma, |\alpha|) \in \W{2}$.
\end{lemma}

\begin{proof}
	We only prove the statement for $\Gamma$ that is strict\,ess.pos. The other case is proven analogously. Since we consider only finite constraint languages we have that $\Gamma \subseteq \IS{r}{02}$ for some $r\geq 2$. Therefore, any $\Gamma$-formula can be written as a conjunction of positive or negative literals and positive clauses of size at most $r$.
	
	Let $\Delta = \{e_1, \dots, e_n\}$ and $\alpha = a_1 \land \dots \land a_k$, where each $a_i$ is either a positive or negative literal, or a positive clause of size $\leq r$.
	Then we have the following three claims.
	\begin{claim}\label{lemW2-claim1}
		If $a_i$ can be explained at all, then at most $r$ formulas from $\Delta$ are sufficient. 
		Consequently, there is a support for $\alpha$ iff there is a support of size at most $r\cdot k$.	
	\end{claim}
	\begin{proof}[Proof of Claim~\ref{lemW2-claim1}]
		If $a_i$ is a negative literal, then at most $1$ formula from $\Delta$ is sufficient (one that contains $a_i$).
		If $a_i$ is a positive literal, then at most $r$ formulas from $\Delta$ are sufficient (worst case: one $e\in\Delta$ contains a (positive) clause which contains $a_i$, then we need at most $r-1$ more $e$'s in order to force all other variables in the clause to $0$).
		If $a_i$ is a positive clause, then it suffices to explain one variable from that clause, as a result, as in the previous case, at most $r$ formulas from $\Delta$ are sufficient. 
		In other words, there is no $a_i$ in $ \alpha$ such that more than $r$ formulas from $\Delta$ are necessary to explain $a_i$.	
	\end{proof}
	\begin{claim}\label{lemW2-claim2}
		Let $\Phi \subseteq \Delta$. If all subsets of $\Phi$ of size $r+1$ are consistent, then $\Phi$ is consistent. In contra position: If $\Phi$ is inconsistent, then it contains a subset of size at most $r+1$ which is inconsistent.
	\end{claim}
	\begin{proof}[Proof of Claim~\ref{lemW2-claim2}]
		Similar to the previous proof: the worst case to create an inconsistency is to take a formula containing a positive clause of size $r$ and then $r$ formulas forcing together all variables from the positive clause to 0.
		\\
		Let $U = \{u_1, \dots, u_m\}$ be a collection of fresh variables, where each variable will represent a different subset of $\Delta$ of size at most $r$, that is, $m \leq r\cdot |\Delta|^r$. For each $u_i$ denote by $S(u_i)$ the subset of $\Delta$ it represents. For $V \subseteq U$ define $S(V) = \bigcup_{u_i \in V} S(u_i)$.
		Define 
		$$L_i = \bigvee_{S(u_j) \models a_i} u_j$$
		and 
		$$\varphi = \bigwedge_{i=1}^k L_i \land \bigwedge_{V \subseteq U s.t. |V| \leq r+1 \text{ and } S(V) \models \emptyset} \Big(\bigvee_{u_i \in V} (\neg u_i)\Big)$$
		The role of each $L_i$ is to make sure that each $a_i$ is explained. The role of the negative clauses in $\varphi$ is to make sure that inconsistent explanations are forbidden.
	\end{proof}
	\begin{claim}\label{lemW2-claim3}
		$(\Delta, \alpha)$ admits a support iff $\varphi$ is satisfiable iff $\varphi$ has a model of weight at most $k$.
	\end{claim}
	\begin{proof}[Proof of Claim~\ref{lemW2-claim3}]
		Be $\Phi \subseteq \Delta$ a support for $\alpha$. 
	Since $\Phi$ explains each $a_i$, by observation 1 there is a set $E(a_i) \subseteq \Phi$ of size at most $r$ such that $E(a_i) \models a_i$. By construction each $E(a_i)$ corresponds to a $u_j \in U$. One verifies that these $u_j$'s constitute a model of $\varphi$ of weight at most $k$ (if $i\neq j$ it can happen that $E(a_i) = E(a_j)$, therefore \emph{at most}).
	
	For the other direction let $\varphi$ be satisfiable. By construction of $\varphi$ there is a model of weight at most $k$ (in each $L_i$ it is sufficient to have at most one positive literal).
	Be $W \subseteq U$ such a model of weight at most $k$. By construction of $\varphi$ the set $S(W)$ is consistent and explains each $a_i$. Therefore, $S(W)$ constitutes a support for $\alpha$.
	\end{proof}
	
	Our final map is $(\Delta, \alpha) \mapsto \varphi \land \bigvee_{i=1}^{k + 1} x_i$, where $x_i$ are fresh variables. We conclude by observing that $\varphi$ is satisfiable if and only if $\varphi \land \bigvee_{i=1}^{k + 1} x_i$ is $(k + 1)$-satisfiable.
	\end{proof}

\begin{theorem}
$\p\ARGcheck(\Gamma,|\alpha|)$, for a CL $\Gamma$, is
(1.) $\FPT$ if $\Gamma$ is Schaefer, and (2.) $\para\DP$-complete otherwise.
\end{theorem}
\begin{proof}
	\begin{enumerate}
		\item This follows from \cite[Theorem~6.1]{DBLP:journals/tocl/CreignouE014} as classically $\ARGrel(\Gamma)\in\P$.
		\item Here, the membership follows as classically $\ARGrel(\Gamma)\in\DP$.
		Furthermore, the reduction in the proof of \cite[Propositions~6.3 and 6.4]{DBLP:journals/tocl/CreignouE014} always uses a fixed size of the claim $\alpha$. 
		As a consequence, certain slices of $\ARGcheck(\Gamma)$ are $\DP$-hard, giving the desired results.\qedhere
	\end{enumerate}
\end{proof}
\begin{theorem}
$\p\ARGrel(\Gamma,|\alpha|)$, for a CL $\Gamma$, is
\begin{enumerate}
	\item $\FPT$ if $\Gamma$ is either positive or negative. 
	\item $\para\NP$-complete if $\Gamma$ is Schaefer but neither strictly ess.neg.\ nor strictly ess.pos.  
	\item $\para\SigmaP$-complete if $\Gamma$ is not Schaefer.
\end{enumerate}
\end{theorem}
\begin{proof}
	\begin{enumerate}
		\item This follows as classically $\ARGrel(\Gamma)\in\P$ by \cite[Prop.~7.3]{DBLP:journals/tocl/CreignouE014}.
	
		\item Here, the membership follows because the classical problem is in $\NP$. 
		We make a case distinction as whether $\Gamma$ is $\varepsilon$-valid or not.
		
		\begin{description}
			\item[Case 1.]
			Let $\Gamma$ be Schaefer and $\varepsilon$-valid, but neither positive nor negative.
			The hardness follows because the $2$-slice of the problem is already $\NP$-hard \cite[Proposition~7.6]{DBLP:journals/tocl/CreignouE014}.
			
			\item[Case 2.] Let $\Gamma$ be Schaefer but neither $\varepsilon$-valid, nor strictly ess.neg.\ or strictly ess.pos
			The hardness follows from Theorem~\ref{theorem:arg-alpha}.
			This is due to the reason that $\p\ARGrel$ is always harder than $\p\ARG$ via the reduction $(\Delta,\alpha) \mapsto (\Delta\cup\{\psi\},\psi,\alpha)$. 
		\end{description}
		
		\item In this case, the membership is true because the classical problem is in $\SigmaP$.
		Hardness follows from a result of \cite[Prop.~7.7]{DBLP:journals/tocl/CreignouE014}. 
		Notice that, while proving the hardness for each sub case, the claim $\alpha$ has fixed size in each reduction. 
		This implies that certain slices in each case are $\SigmaP$-hard, consequently, giving the desired hardness results.\qedhere
	\end{enumerate}
\end{proof}

\section{Parameters: Size of Support, Knowledge-Base}
Regarding these parameters, we will always show a dichotomy: for the Schaefer cases, the problem is $\FPT$, otherwise we have a lower bound by the implication problem.

Recall that the collection $\Delta$ of formulas is not assumed to be consistent. 

\begin{theorem}\label{theorem:Delta}
	$\p\ARG(\Gamma,|\Delta|)$ and $\p\ARGrel(\Gamma,|\Delta|)$, for CLs $\Gamma$, are (1.) $\FPT$ if $\Gamma$ is Schaefer, and (2.) $\p\IMP(\Gamma,|\Phi|)$-hard and in $\para\co\NP$ otherwise.
\end{theorem}

\begin{proof}
	\begin{enumerate}
		\item Notice that the number of subsets of $\Delta$ is bounded by the parameter.
		Consequently, one simply checks each subset of $\Delta$ as a possible support $\Phi$ for $\alpha$.
		Moreover, the size of each support $\Phi$ is also bounded by the parameter, as a result, one can determine the satisfiability and entailment in $\FPT$-time.
		This is because, the satisfiability and entailment for Schaefer languages is in $\P$.
		\item For the lower bound, we have $\p\IMP(\Gamma,|\Phi|)\fptreduction\ARG(\Gamma,|\Delta|)\fptreduction\p\ARGrel(\Gamma,|\Delta|)$ by identities.
		
		For membership, we make case distinction as whether $\Gamma$ is $\varepsilon$-valid or not.
		\begin{description}
			    \item[Case 1.] $\Gamma$ is $\varepsilon$-valid.
			The membership follows because the unparameterized problem $\ARG(\Gamma)$ is in $\co\NP$ when $\Gamma$ is $\varepsilon$-valid.
			
				\item[Case 2.] $\Gamma$ is neither $0$-valid nor $1$-valid.
			The membership follows because for each candidate $\Phi$, one needs to determine whether $\Phi$ is consistent and $\Phi\models\alpha$. 
			The consistency can be checked in $\FPT$-time because $|\Phi|$ is bounded by the parameter.
			The entailment problem for non-Schaefer, non $\varepsilon$-valid languages is still in $\para\co\NP$ when $|\Phi|$ is the parameter.		
		\end{description}
		For $\p\ARGrel(\Gamma,|\Delta|)\in\para\co\NP$, try all the subsets of $\Delta$ that contain $\psi$, as a candidate support. \qedhere
	\end{enumerate}
\end{proof}





When the support size $|\Phi|$ is considered as a parameter, the problems $\ARG$ and $\ARGrel$ become irrelevant. 
Consequently, we only consider the problem $\ARGcheck$.
\begin{corollary}\label{cor:argcheck-phi}
	$\p\ARGcheck(\Gamma, |\Phi|)$, for a CL $\Gamma$, is
(1.) $\FPT$ if $\Gamma$ is Schaefer, and (2.) $\p\IMP(\Gamma,|\Phi|)$-hard and in $\para\DP$ otherwise.
\end{corollary}


\section{Conclusion and Outlook}
In this paper, we performed a two dimensional classification of reasoning in logic-based argumentation.
On the one side, we studied syntactical fragments in the spirit of Schaefer's framework of co-clones.
On the other side, we analysed a list of parameters and classified the parameterized complexity of three central reasoning problems accordingly.

As a take-away message we get that $\alpha$ as a parameter does not help to reach tractable fragments of $\p\ARG$.

The case for $\p\ARGrel(\Gamma, |\alpha|)$ when $\Gamma$ is strictly ess.neg.\ or strictly ess.pos.\ is still open.
Also, few tight complexity results have to be found and the implication problem regarding the parameter $|\Phi|$ has to be understood.

It is worth noting that for some CLs, e.g., those that are $\varepsilon$-valid, the problem $\p\ARGcheck$ is harder than $\p\ARG$. 
This is because the problem $\p\ARG$ under consideration is the decision problem.
Having the identity reduction from $\p\ARGcheck$ to $\p\ARG$ shows that the minimality is checked by solving the problem $\p\ARG$, already.
This shows that computing a minimal support is potentially harder than deciding whether such a support exists, unless the complexity classes $\DP$ and $\co\NP$ coincide.
We pose as an interesting open problem to classify the function version of $\ARG$, in both, the classical and the parameterized setting. 

Regarding other parameters, treewidth~\cite{DBLP:journals/jct/RobertsonS84} is a quite promising structural property that led to several $\FPT$-results in the parameterized setting: artificial intelligence~\cite{GottlobSzeider07}, knowledge representation~\cite{GottlobPichlerWei06}, abduction in Datalog~\cite{GottlobPichlerWei07}, and databases~\cite{Grohe07}.
Fellows~et~al.~\cite{DBLP:conf/aaai/FellowsPRR12} show that abductive reasoning benefits from this parameter as well.
Using a reduction between abduction and argumentation \cite{DBLP:journals/tocl/CreignouE014} might yield $\FPT$-results in our setting.
Furthermore, we plan to give a precise classification of $\p\IMP$.

As further future work, we plan investigating the (parameterized) enumeration complexity \cite{DBLP:series/txtcs/FominK10,DBLP:journals/mst/CreignouMMSV17,DBLP:journals/algorithms/CreignouKMMOV19,DBLP:books/dnb/MeierM20} of reasoning in this setting.

\bibliography{main.bib}
\end{document}